\documentclass{article}
\usepackage{kr}

\usepackage{times}
\usepackage{soul}
\usepackage{url}
\usepackage[utf8]{inputenc}
\usepackage[small]{caption}
\usepackage{graphicx}
\usepackage{amsmath}
\usepackage{amsthm}
\usepackage{booktabs}
\usepackage{algorithm}
\usepackage{algorithmicx}
\usepackage{algpseudocode}
\usepackage[table]{xcolor}
\usepackage{amssymb}
\usepackage{relsize}
\usepackage{tikz}
\usepackage{enumerate}
\usepackage{booktabs}
\usepackage{xspace}
\usepackage{fancyvrb}
\usepackage{fvextra}

\newcommand{\Omit}[1]{}

\newcommand{\tup}[1]{\ensuremath{\langle #1 \rangle}}
\newcommand{\set}[1]{\ensuremath{\{ #1 \}}}

\newcommand{\mbracket}[1]{\ensuremath{[\![ #1 ]\!]}}
\renewcommand{\O}{\mathcal{O}}
\newcommand{\citeau}[1]{\citeauthor{#1}\xspace(\citeyear{#1})\xspace}

\newcommand{\genex}{\textsc{\small GenEx}\xspace}
\newcommand{\wrapper}{\textsc{\small Wrapper}\xspace}
\newcommand{\dlplan}{\textsc{\small Dlplan}\xspace}
\newcommand{\tarski}{\textsc{\small Tarski}\xspace}
\newcommand{\siw}{\textsc{\small SIW}\xspace}
\newcommand{\siwr}{\textsc{\small SIWR}\xspace}
\newcommand{\bfws}{\textsc{\small BFWS}\xspace}

\renewcommand{\S}{\mathcal{S}}
\newcommand{\T}{\mathcal{T}}
\newcommand{\Q}{\mathcal{Q}}

\newcommand{\F}{\mathcal{F}}
\newcommand{\G}{\mathcal{G}}

\renewcommand{\H}{\mathcal{H}}

\newcommand{\Th}{\text{Th}}
\newcommand{\BLT}{\text{BLT}}
\newcommand{\DOMAIN}[1]{{\small\Verb!#1!}\xspace}
\newcommand{\Ord}{\text{Ord}}

\newcommand{\EQ}[1]{\ensuremath{#1\,{=}\,0}}
\newcommand{\GT}[1]{\ensuremath{#1\,{>}\,0}}
\newcommand{\INC}[1]{\ensuremath{#1\raisebox{1.10pt}{\smaller$\uparrow$}}}
\newcommand{\DEC}[1]{\ensuremath{#1\raisebox{1.10pt}{\smaller$\downarrow$}}}
\newcommand{\UNK}[1]{\ensuremath{#1?}}
\newcommand{\arule}[2]{\ensuremath{#1\,{\mapsto}\, #2}}
\newcommand{\prule}[2]{\arule{\set{#1}}{\set{#2}}}

\newcommand{\cost}{\text{cost}}

\newcommand{\cond}{\text{cond}}
\newcommand{\eff}{\text{eff}}
\newcommand{\sieve}{\textsc{Sieve}\xspace}

\newcommand{\X}{\mathcal{X}}
\newcommand{\FAIL}{\textsc{Failure}\xspace}

\newtheorem{definition}{Definition}
\newtheorem{theorem}[definition]{Theorem}
\newtheorem*{example*}{Example}

\newcommand{\hector}[1]{\textcolor{red}{Hector: #1}}

\title{Learning General Policies From Examples}
\Omit{
}
\author{%
  Blai Bonet$^1$\and
  Hector Geffner$^2$ \\
  \affiliations
  $^1$Universitat Pompeu Fabra, Spain \\
  $^2$RWTH Aachen University, Germany \\
  \emails
  bonetblai@gmail.com, hector.geffner@ml.rwth-aachen.de
}

\begin{document}
\allowdisplaybreaks
\maketitle

\begin{abstract}
  Combinatorial methods for learning general policies that solve
  large collections of planning problems have been recently developed.
  One of their strengths, in relation to deep learning approaches, is
  that the resulting policies can be understood and shown to be correct.
  A weakness is that the methods do not scale up
  and learn only from small training instances and feature pools
  that contain a few hundreds of states and features at most.
  In this work, we propose a new symbolic method for learning
  policies based on the generalization of sampled plans
  that ensures structural termination and hence acyclicity.
  The proposed learning approach is not based on SAT/ASP, as previous
  symbolic methods, but on a hitting set algorithm that can effectively handle problems
  with millions of states, and pools with hundreds of thousands of features.
  The formal properties of the approach are analyzed, and its scalability
  is tested on a number of benchmarks.
\end{abstract}

\section{Introduction}
\label{sect:introduction}

The problem of learning policies that solve large collections of planning problems
is an important challenge in both planning and reinforcement learning.
Symbolic methods yield general policies that can be shown to be correct
but are limited to small training instances and feature pools containing
at most hundreds of states and features \cite{frances:aaai2021}. Deep learning methods,
on the other hand, do not require feature pools and scale up gracefully,
yet the resulting policies are opaque and, in general, do not generalize
equally well \cite{sylvie:asnet,stahlberg-et-al-kr2022}.

The aim of this work is to develop a different way of learning
general policies in the symbolic setting that scales up to much larger
training instances and feature pools, including millions of states
and hundreds of thousands of features. This scalability is also needed to
address a limitation that is shared by the symbolic and deep learning
approaches, and which has to do with the type of state features that
can be computed.
When description logic grammars or graph neural networks are used, the
only logical features that can be captured are those that can be defined
in first-order logic with two variables and
counting quantifiers \cite{barcelo:gnn,grohe:gnn}. %\footnote{\alert{*** A veviewer notes: that GNNs have larger scope; also, distance features in the pool are outside C2 ***}}
Addressing this limitation in the symbolic setting, calls for
novel and richer feature grammars that result in larger feature
pools, thus requiring more scalable learning algorithms.

Previous symbolic methods like \cite{frances:aaai2021,drexler:icaps2022} scale up poorly because they
reduce the task of learning the features and the policy
%a policy, both the learning of the required features and the learning of the policy itself,
to a combinatorial optimization
problem that is cast and solved by weighted-SAT or ASP solvers.
In these settings, relaxing the optimality criterion, or some of the
constraints, yield policies that do not generalize well outside the
training set.
In this work, we use an scalable, heuristic algorithm for min-cost
hitting set problems as the basis of a new procedure for learning
general policies.
%a novel, efficient, and powerful class
%of target policies for a pool of features, and then reduce the problem
%of finding such a policy to a simpler \emph{min-cost hitting set problem}
%that can be solved efficiently with approximation algorithms.

For obtaining the new formulation, two ingredients are needed.
First, a classical planner that generates plans which are then generalized.
Second, a new structural termination criterion that ensures that the
generalization does not introduce cycles.
%First, a classical planner that is called as an \emph{oracle} to
%generate plans that are then generalized. As we will see, few calls
%to the planner are needed.
%Second, a \emph{new structural termination criterion} is introduced
%that ensures the generalization of the plans cannot yield cyclic
%executions, and that is used to define the class of target policies
%for learning.
Provided with this guarantee, the plan generalizations become fully
general policies when they are also \emph{closed} and \emph{safe};
meaning that they do not reach states where the policy is undefined
or which are dead ends, respectively.
The resulting algorithm consists of an efficient, polynomial-time,
core algorithm, based on min-cost hitting sets, that yields a policy
that generalizes given sets of positive and negative state transitions,
$\X^+$ and $\X^-$, and a wrapper algorithm that manipulates these two
sets until the resulting policy is closed and safe, and hence correct.
Interestingly, when the algorithm \emph{fails} to produce a correct policy,
the reasons for the failure can be understood, and sometimes, fixed.

The rest of the paper is structured as follows.
We discuss related work, cover relevant
background, introduce the new termination
criterion, and present the resulting formulation
and algorithm, assessing its properties and
its performance.

\section{Related Work}
\label{sect:related}

\noindent\textbf{General policies.} The problem of learning general policies has a long history
\cite{khardon:action,fern:generalized,srivastava08learning,hu:generalized,BelleL16,sheila:generalized2019,sergio:generalized}.
General policies have been formulated in terms of first-order logic \cite{srivastava:generalized,sheila:generalized2019},
first-order regression \cite{boutilier2001symbolic,wang2008first,sanner:practicalMDPs},
and neural networks \cite{sid:sokoban,trevizan:dl,sanner:dl,erez:generalized,mausam:dl2,simon:kr2023}.
Our work builds on formulations where the features and policies defined on such features and rules are learned using SAT
encodings \cite{bonet:ijcai2018,frances:aaai2021}, and is also related to early supervised approaches that use polynomial
algorithms and explicit feature pools \cite{martin:generalized}.

\medskip

\noindent \textbf{QNPs, Termination, Acyclicity.} While a policy that is closed, safe, and acyclic must solve a problem,
enforcing acyclicity is not easy computationally, as it is not a local property. Interestingly, structural criteria and
algorithms that ensure acyclicity have been developed in the setting of qualitative numerical planning problems or
QNPs \cite{sid:aaai2011,bonet:qnps}, which involve Boolean and numerical variables that can be increased and decreased by
random amounts.
A policy \emph{terminates} in a QNP if all ``fair'' trajectories are finite, where trajectories where a variable increases
\emph{finitely} often and decreases infinitely often are regarded as ``unfair'', and can be ignored.
An algorithm called \sieve establishes termination in time that is exponential in the number of variables in the QNP.
%The notion of termination is also used in compilation of QNP problems into fully-observable non-deterministic (FOND) problems \cite{bonet:qnps,ivan:fond-asp}.
In this work a new termination criterion is introduced that is slightly weaker than \sieve but that is more convenient and can be built into the selection of the features.

\medskip

\noindent \textbf{Imitation learning and inverse reinforcement learning (IRL).} The use of sampled plans for learning general policies
has been used in early work \cite{khardon:action,martin:generalized,fern:generalized}, and its a common idea in imitation learning \cite{ng-russell:inverse-rl,ho-ermon:il}.
Learning to imitate plans or behavior in a mindless manner, however, prevents good generalization. The idea in IRL is to learn the reward distribution
from the examples and then solve the underlying problem where these rewards are to be optimized. The problem in IRL is in the assumptions that need
to be made so that the resulting task is well-posed. In this work, we also aim to go beyond the examples (plans) and solve the more general problem
that the examples illustrate. The task is not just to generalize the given plans, but to obtain a policy that is
structurally terminating and which hence must converge to a goal.

\section{Background}
\label{sect:background}

We review classical planning, generalized planning,
rule-based policies, learning these policies, and termination,
following \cite{frances:aaai2021,bonet:jair=2024}.

\subsection{Planning and Generalized Planning}
\label{sect:background:planning}

We deal with planning instances $P\,{=}\,\tup{D,I}$ where $D$
is a \textbf{domain} specification containing object types, constants, predicate
signatures, and action schemas, and $I$ is an \textbf{instance} specification
containing the objects and their types, and the description of the initial
and goal situations, $I$ and $G$ respectively, both as sets of ground atoms.

A state trajectory in $P$ seeded at $s_0$ is a state sequence $\tau\,{=}\,\tup{s_0,s_1,s_2,\ldots}$
such that for each transition $(s_i,s_{i+1})$, there is a ground
action that maps $s_i$ to $s_{i+1}$.
A state $s$ is \textbf{reachable} iff there is a trajectory seeded at the initial state that ends in $s$;
it is a \textbf{dead end} iff it is reachable, and there is no state trajectory seeded at $s$ that ends in a goal state;
and it is \textbf{alive} iff it is reachable, and it is not a goal nor a dead-end state.
The task for $P$ is to find a trajectory seeded at the initial state that ends in a goal state,
or declare that no such trajectory exist; a sequence of ground actions that
map each state into the next in such a trajectory is a \textbf{plan.}

Semantically, a policy $\pi$ for $P$ is a \textbf{set} of state transitions $(s,t)$ in $P$.
A $\pi$-trajectory from state $s_0$ is a state trajectory $\tau\,{=}\,\tup{s_0,s_1,s_2,\ldots}$
such that $(s_i,s_{i+1})$ is in $\pi$. %, for $i\,{=}\,0,1,2,\ldots$.
The states in the trajectory are said to be $\pi$-reachable in $P$.
The policy $\pi$ \textbf{solves} $P$ iff each \textbf{maximal} $\pi$-trajectory from
the initial state is finite and ends in a goal state.
The policy $\pi$ is \textbf{closed} if for each alive state $s$ that is $\pi$-reachable, there is $\pi$-transition $(s,s')$;
it is \textbf{safe} if it does not reach a dead end; and
it is \textbf{acyclic} if there is no \textbf{infinite} $\pi$-trajectory seeded at the initial state.
The following characterizes policies that solve $P$:

\begin{theorem}[Solutions for $P$]
  \label{thm:solutions:P}
  Let $P$ be a planning instance, and let $\pi$ be a policy for $P$.
  Then, $\pi$ solves $P$ iff $\pi$ is closed, safe, and acyclic in $P$.
\end{theorem}

\subsection{Generalized Planning}
\label{sect:background:generalized}

Generalized planning deals with the computation of policies that solve
collections of planning instances rather than just a single planning
instance.
A collection of planning instances is a, finite or infinite, set $\Q$
of instances $P_i\,{=}\,\tup{D,I_i}$ over a common domain $D$.
In some cases, all the instances in $\Q$ have the same or similar goal,
like achieving a specific atom, but this is not required, nor assumed.

Semantically, a policy $\pi$ for $\Q$ represents a subset of state transitions in each instance $P$ in $\Q$.
The policy $\pi$ solves $\Q$ iff it solves each instance $P$ in $\Q$.
Notions about states like reachable, alive, goal, etc., and about
policies like closed, safe, etc., are naturally lifted from $P$
into $\Q$. A similar characterization for solutions for $\Q$ applies:

\begin{theorem}[Solutions for $\Q$]
  \label{thm:solutions:Q}
  Let $\Q$ be a collection of planning instances, and let $\pi$ be a policy for $\Q$.
  Then, $\pi$ solves $\Q$ iff $\pi$ is closed, safe, and acyclic in $\Q$.
\end{theorem}

\subsection{Features}
\label{sect:features}

General policies can be represented in terms of collections of feature-based rules.
Features are functions that map states into values.
% A feature for a planning instance $P$ is a function $f$ that maps states in $P$ into values in a domain.
Boolean features take values in \set{0,1}, and numerical features take
values in the non-negative integers. In logical accounts, features are commonly defined in terms of
concepts (unary predicates), which are \textbf{generated} from sets of \textbf{atomic} concepts and roles,
using description logic grammars \cite{martin:generalized,fern:generalized,bonet:aaai2019}.
The \textbf{denotation} of concept $C$ (resp.\ role $R$) in a state $s$ is a set of objects (resp.\ pairs
of objects) from $s$, denoted by $C(s)$ (resp.\ $R(s)$).
Concept $C$ defines a numerical feature $f_C$ whose value at $s$ is the cardinality $|C(s)|$ of the set $C(s)$.
When the value of $f_C$ is always in \set{0,1}, the concept defines a Boolean rather than a numerical feature.

For a domain $D$, the atomic concepts (resp.\ roles) are given by the object
types, constants, and unary predicates (resp.\ the binary predicates) in $D$.
A \textbf{pool of features} $\F$ can be generated using the domain $D$, and parameters
that bound the maximum complexity and depth for the features in $\F$.
The generation process is given a set $\T$ of transitions over instances in $D$
to \textbf{prune redundant} features; namely, if $\S$ is the set of
states mentioned in the transitions in $\T$, a feature $f$ is redundant if there is
a feature $g$ of lesser complexity, or equal complexity but earlier in a static ordering,
such that both are $\S$-equivalent, or $\T$-equivalent.
Feature $f$ is $\S$-equivalent to $g$ iff $f(s)\,{=}\,g(s)$ for each state $s$ in $\S$,
and $f$ is $\T$-equivalent to $g$ if for each $(s,t)$ in $\T$, both have the
same Boolean valuation at $s$, and both change equivalently across $(s,t)$ (i.e.,
$\GT{f(s)}$ iff $\GT{g(s)}$, and both increase/decrease/stay equal across $(s,t)$).

In the generalized planning setting, states $s$ for an instance $P$ are assumed
to contain the description of the goal in $P$ via \textbf{goal predicates} $p_g$ \cite{martin:generalized},
one for each predicate $p$ in $D$ with denotation $\set{ \bar u \mid G \vDash p(\bar u)}$.
These predicates allow policies to work for different goals $G$.

\begin{example*}
  Let us consider the domain for Blocksworld with 4 operators (i.e., with a gripper)
  and an instance $P$ whose goal description is $G\,{=}\,\set{clear(A)}$.
  The domain description contains the predicates $clear/1$ and $on/2$. %\set{holding/1, clear/1, ontable/1, on/2}.\
  The following concepts are generated by the grammar:
  \begin{enumerate}[--]
    \item `$\text{clear}_g$' whose denotation is the singleton with block $A$,
    %\item `$\text{ontable}$' whose denotation is the set of blocks directly on the table,
    \item `$\exists \text{on}.\top$' whose denotation consists of the blocks that rest on another block ($\top$ is the concept that includes all objects), and
    \item `$\exists \text{on}^+.\text{clear}_g$' whose denotation consists of the blocks that are above block $A$.
  \end{enumerate}
  The second concept defines a feature that counts the number of blocks that rest on another block,
  while the third defines a feature that counts the number of blocks above the ``target'' block $A$.
  \qed
\end{example*}

For a set $\G$ of features, a \textbf{Boolean valuation} $\nu$ is a function $\nu\,{:}\,\G\,{\rightarrow}\,\set{0,1}$
that assigns a Boolean value to all the features in a state, whether Boolean or numerical,
as $\EQ{\nu(f)}$ if $\EQ{f(s)}$, and $\nu(f)\,{=}\,1$ if $\GT{f(s)}$. In these valuations,
the exact value of numerical features is abstract away, replaced by a Boolean that is true
iff the value is strictly positive.

% The valuation is ``attached'' to a state $s$ iff $\EQ{\nu(f)}$ if $\EQ{f(s)}$, and $\nu(f)\,{=}\,1$ if $\GT{f(s)}$.

\subsection{Rule-based Policies}
\label{sect:background:policies}

While semantically, policies select subset of state transitions in each instance $P$ in $\Q$,
syntactically, policies are represented by collection of \textbf{rules} over a given set of features $\F$.
A policy rule \arule{C}{E} consists of a condition $C$, and an effect $E$,
where $C$ contains expressions like $p$ and $\neg p$ for Boolean features $p$,
and $\EQ{n}$ and $\GT{n}$ for numerical features $n$.
The effect $E$ in turn contains expressions like $p$, $\neg p$, and $\UNK{p}$
for Booleans $p$, and $\INC{n}$, \DEC{n}, and \UNK{n} for numericals $n$.

A set of rules defines a policy $\pi$ where a state transition $(s,t)$
is a $\pi$-transition if it is compatible with one of the rules $r=\arule{C}{E}$.
This is true if the feature conditions in $C$ are true in $s$, and the
features change in the transition in a way that is compatible with $E$;
i.e., $(s,t)$ is compatible with $r$ iff

\begin{enumerate}[--]
  \item if $p$ (resp.\ $\neg p$) in $C$, $p(s)$ (resp.\ $\neg p(s)$),
  \item if $p$ (resp.\ $\neg p$) in $E$, $p(t)$ (resp.\ $\neg p(t)$),
  \item if $\set{p,\neg p,\UNK{p}}\cap E=\emptyset$, $p(s)$ iff $p(t)$,
  \item if $\EQ{n}$ (resp.\ $\GT{n}$) in $C$, $\EQ{n(s)}$ (resp.\ $\GT{n(s)}$),
  \item if $\INC{n}$ (resp.\ $\DEC{n}$) in $E$, $n(s)\,{<}\,n(t)$ (resp.\ $n(s)\,{>}\,n(t)$), and
  \item if $\set{\INC{n},\DEC{n}p,\UNK{n}}\cap E=\emptyset$, $n(s)\,{=}\,n(t)$.
\end{enumerate}
The transition $(s,t)$ is in $\pi$ if it is compatible with some rule $r$ in $\pi$;
we use notations like $(s,t)\,{\in}\,\pi$, and $\set{(s,t)}\subseteq\pi$.

Feature-based rules permit the representation of general policies that are
not tied to a particular set of instances, as those used for learning such a policy,
as the features and rules can be evaluated on any instance for the shared domain.

\begin{example*}
  A general policy for Blocksworld instances with a gripper and goal
  descriptions of the form $G\,{=}\,\set{clear(A)}$, for arbitrary
  block $A$, can be expressed with only two rules \prule{\neg H, \GT{n}}{H, \DEC{n}}
  and \prule{H}{\neg H}, where $H$ is a Boolean feature that tells
  whether a block is being held, and $n$ counts the number of blocks
  above the target block (i.e., the block mentioned in $G$).
  The first rule says to pick block above the target when holding nothing,
  while the second to put the block being held somewhere \emph{not}
  above the target (this last condition is achieved as the effect
  of the second rule requires $n$ to be remain constant).

  Notice, however, that if the second rule is replaced by \prule{H}{\neg H, \UNK{n}},
  the resulting policy does not solve such instances as it can
  generate infinite trajectories where the same block is picked
  and put back above the target block, repeatedly.
  The policy above cannot generate infinite
  trajectories, independently of the interpretation of $H$ and $n$;
  we say that it is \emph{structurally terminating.}
  \qed
\end{example*}

\subsection{Learning Rule-based Policies}
\label{sect:background:learning}

The new method for learning general policies builds on the SAT approach
developed by \citeau{frances:aaai2021} that constructs a CNF theory
$\Th(\T,\F)$ from a set $\T$ of state transitions $(s,t)$ in a collection
$\Q'$ of training instances, and a pool $\F $ of Boolean and numerical features.
The theory contains propositions $good(s,t)$ and $select(f)$ that tell
which transitions $(s,t)$ in $\T$ and features $f$ in $\F$ should be
included in the resulting policy $\pi$.
The policies $\pi$ over the features in $\F$ that solve $\Q'$ are in
\emph{correspondence} with the models of $\Th(\T,\F)$.
Indeed, the rules in $\pi$ are determined by the true
$good(s,t)$ and $select(f)$ atoms: each good transition $(s,t)$ yields
a rule \arule{C}{E} where $C$ captures the Boolean valuations of the
selected features at state $s$, and $E$ captures the changes of such
features across $(s,t)$.
Such a policy rule is called the \textbf{projection} of the transition
$(s,t)$ over the set of selected features.

By constructing a policy $\pi^*$ from a satisfying assignment of
\emph{minimum cost}, as determined by the complexity of the selected features,
the policy $\pi^*$ is expected to generalize over the entire class $\Q$
from which the training instances in $\Q'$ have been drawn.
While correct generalization is not guaranteed, this can be established
by manually analyzing $\pi^*$.

The key constraints in the SAT theory, expressed with the $good(s,t)$ and $select(f)$
atoms, ensure that any resulting policy $\pi$ is \textbf{closed}, \textbf{safe},
and \textbf{acyclic} in $\Q'$.
For this, the \emph{full state space} of such instances, states and transitions,
need to be calculated and represented in $\Th(\T,\F)$ as well as all the
features in the pool, making the approach only feasible for small state spaces
and feature pools.

\subsection{Termination}
\label{sect:background:learning:termination}

The idea of structural termination in settings where variables can be increased and decreased
``qualitatively'' (i.e., by random amounts) is that certain state trajectories are not possible
if the variables have minimal and maximal values in each instance, and the changes cannot
be infinitesimally small \cite{sid:aaai2011,bonet:qnps}.
In the generalized planning setting, it is the numerical features that change in this way.
An infinite state trajectory where a numerical feature is increased \emph{finitely} often,
and decreased \emph{infinitely} often, or vice versa,
%where it is decreased finitely often and increased infinitely often,
is not possible.
If the infinite state trajectories generated by a policy all have this form,
the policy must be \emph{acyclic.}

This notion of termination is particularly interesting because it comes up with a sound
and complete algorithm for checking the termination of a given policy $\pi$,
called \textsc{sieve} \cite{sid:aaai2011} that runs in time $\O(2^n)$ where
$n$ is the number of features in $\pi$, and which works on the so-called
\emph{policy graph,} whose nodes are the possible Boolean valuations of the features in $\pi$.
Indeed, while acyclicity is property of the policy $\pi$ in a particular instance $P$,
termination is a property of the rules in $\pi$ that ensures acyclicity
for \emph{any} instance. The notion of \emph{stratified rule-based policies}
introduced below provides a novel twist to this idea which is more convenient
for learning policies that are \textbf{terminating by design.} Indeed, the new
structural termination criterion runs more efficiently than \textsc{sieve}
and can be compiled into the procedure that selects the policy features
given the good transitions, ensuring that the resulting policy
is terminating.

\section{The Plan}

The approach for learning general policies via SAT reviewed above % from \citeau{frances:aaai2021}
is simple and elegant but does not scale up. %, and it can only deal with sets $\S$ and $\F$
%containing a few hundred state transitions and features at most.
The new learning method %to be developed below
can be thought as a simplification where: %of this general approach where:

\begin{enumerate}[1.]
  \item The ``good'' state transitions are not selected by the SAT solver but incrementally, by a \emph{planner}.
  \item The acyclicity constraint, which is global and hard to enforce, is replaced by a \emph{new termination criterion} that %ensures acyclicity and
    is enforced implicitly in the selection of the features.
  \item The selection of the features is carried out by a scalable \emph{hitting set algorithm} and not by SAT.
\end{enumerate}

The three elements are all critical for the performance of the resulting learning algorithm.
The second and third elements are addressed by an algorithm that we call \genex, for \emph{gen}eralization from \emph{ex}amples,
and that can handle sets of state transitions and features that are orders of magnitude larger
that the ones handled by current approaches.
The first element, on the other hand, that uses a planner to select transitions, is addressed
by a simple \wrapper algorithm around \genex.

An essential idea of the new method is the \emph{decoupling} of the two selections involved in the computation of general policies:
the selection of the ``good'' transitions in the training set, and the selection of the features.
In the SAT approach these two decisions are coupled, ensuring completeness: if there is a general policy over the
features that solves the training instances, the SAT approach would find it.
This guarantee is now gone, replaced by the decoupling that enables scalability.
The experiments below evaluate this trade off.

The next sections introduce the new termination criterion, a new basic learning algorithm
that yields terminating policies that generalizes given sets of ``good'' and ``bad'' transitions,
and a wrapping mechanism that extends this policy to be safe and closed. %, by adding more ``good'' and ``bad'' transitions.
%by adding more ``good'' and ``bad'' transitions to the sample set.

\section{Termination Revisited}

\Omit{
  A policy $\pi$ is terminating iff for each infinite
  $\pi$-trajectory, there is a feature $f$ that
  decreases infinitely often in $\tau$, but only increases a finite number of times
  in $\tau$. The existence of $f$ for a given trajectory does not depend on the
  behavior of other features in the trajectory. As it turns out, this is perhaps
  the most general notion of termination available for general feature-based
  policies.
}

% We define a novel and stronger notion of termination, that while having smaller
% scope than \ (general) termination, it is broad, in one hand, and effective, on the other.
A new notion of termination is introduced by means of \textbf{stratified policies}.
These are \textbf{rule-based policies} whose features can be layered up in such a way,
that {features} in the first layer can be shown to be terminating without having to
consider other features, while features in successive layers can be shown to be
terminating given features in previous layers. A \textbf{feature $f$ is terminating}
in a policy $\pi$ if it cannot keep changing values forever; namely, if the number of
times that $f$ changes value in a $\pi$-trajectory is finite. The policy $\pi$
is terminating if all the features involved are terminating.

%
% The notion is parametrized in terms of the max number $k$ of features that need
% to be considered simultaneously to render another feature as ``terminating''.

We take advantage of the assumption that the numerical features
are non-negative integer valued and upper bounded in any instance. %(as the number of states is finite).
Hence, a trajectory where
a feature is increased (resp.\ decreased) infinitely often but decreased (resp.\ increased)
finitely often is not possible.
% However, this assumption can be easily removed to capture the standard definition
% of termination.
In the definitions, Boolean features are treated as numerical features
with decrements and increments referring to value changes from from 1 to 0, and from 0 to 1,
respectively.

\subsection{Stratified Policies}
\label{subsection:stratified}

The first class of terminating features are the features that are \textbf{monotone}
in the set of (policy) rules $R$, meaning that $R$ does not contain rules that
increase and that decrease $f$.

% We first define what it means for a feature $f$ to be monotone for a set of rules $R$,
% and then lift this notion to define conditional monotonicity with respect to another
% feature $g$.

\begin{definition}[Monotone Features]
  \label{def:monotone}
  Let $\F$ be a set of features, let $R$ be a set of policy rules over $\F$,
  and let $f$ be a feature that appears in some rule in $R$.
  Then, $f$ is \textbf{monotone in $R$} iff either there is no rule in $R$
  that increases $f$, or there is no rule in $R$ that decreases $f$.
\end{definition}

Clearly, a monotone feature can only change value a finite number
of times along a given trajectory, as it eventually reaches a
minimum or maximum value, and stays put.

The second class of terminating features $f$ are those that are % iteratively
rendered monotone by other monotone features $g$. Indeed, since $g$ can
change a finite number of times, $f$ will be \textbf{monotone given $g$}
if $f$ is monotone over the rules that do not change the value of $g$
and which share the same value of $g$ in the conditions.

\Omit{
If $\tau$ is an $R$-trajectory (finite or infinite), the \textbf{Boolean value}
of $f$ along $\tau$ cannot flip an infinite number of times. It can only flip at most
once, from $0$ to $1$, or from $1$ to $0$.
}

To capture this form of \textbf{conditional monotonicity}, we define the following
rule subsets from $R$ and a given feature $g$:
\begin{alignat*}{1}
  \varrho(R,g,=)\ &\doteq\ \set{r\,{\in}\,R \mid \set{\INC{g},\DEC{g}}\,{\cap}\,\eff(r)\,{=}\,\emptyset } \,, \\
  \varrho(R,g,0)\ &\doteq\ \set{r\,{\in}\,\varrho(R,g,=) \mid \text{`$\GT{g}$'}\,{\not\in}\, \cond(r)} \,, \\
  \varrho(R,g,1)\ &\doteq\ \set{r\,{\in}\,\varrho(R,g,=) \mid \text{`$\EQ{g}$'}\,{\not\in}\, \cond(r)} \,.
\end{alignat*}

The intuitions for these subsets is that when a transition $(s,t)$ is compatible with a rule $r$, then

\begin{enumerate}[$\bullet$]
  \item $r\,{\in}\,\varrho(R,g,=)$ iff $g$ may remain unchanged across $(s,t)$,
  \item $r\,{\in}\,\varrho(R,g,0)$ iff $r\,{\in}\,\varrho(R,g,=)$ and $\EQ{g}$ may hold at $s$,
  \item $r\,{\in}\,\varrho(R,g,1)$ iff $r\,{\in}\,\varrho(R,g,=)$ and $\GT{g}$ may hold at $s$.
\end{enumerate}

Conditional monotonicity is defined as follows:

\begin{definition}[Conditional Monotonicity]
  \label{def:monotone:conditional}
  Let $\F$ be a set of features, let $R$ be a set of policy rules over $\F$,
  and let $f$ and $g$ be features that are mentioned in $R$.
  Then, $f$ is \textbf{monotone in $R$ given $g$} iff $f$ is monotone in
  $\varrho(R,g,0)$ \textbf{and} $f$ is monotone in $\varrho(R,g,1)$.
\end{definition}

A set of rules $R$ encodes a \textbf{stratified policy} if the features
in the policy can be ordered in such a way that features $f$ with a
positive rank $\kappa(f)$ are monotone given features $g$ of lower rank.

% A ranking $\kappa$ of the features that appear in a set of rules $R$ is an
% assignment of such features to the non-negative integers. We use the notion
% of monotonicity and ranking to define the set of stratified policies:

\begin{definition}[Stratified Policies]
  \label{def:stratified}
  Let $\pi$ be a rule-based policy over a set $\F$ of features.
  Then, $\pi$ is a \textbf{stratified} iff % policy} iff
  \begin{enumerate}[1.]
    \item Each rule in $\pi$ \textbf{entails the change} of some feature $f$, and
    \item There is a ranking $\kappa$ for the features in $\pi$ such that
      \begin{enumerate}[{2}a.]
        \item If $\kappa(f)\,{=}\,0$, then $f$ is \textbf{monotone in $\pi$,} and
        \item If $\kappa(f)\,{>}\,0$, then there is feature $g$ such that $\kappa(g)\,{<}\,\kappa(f)$ and $f$ is \textbf{monotone in $\pi$ given $g$.}
      \end{enumerate}
  \end{enumerate}
  Where a rule \arule{C}{E} \textbf{entails the change} of feature $f$ iff
  $f$ is numerical and $E\cap\set{\INC{f},\DEC{f}}\neq\emptyset$, or $f$ is Boolean
  and either $p\,{\in}\,C \land \neg p\,{\in}\,E$, or $\neg p\,{\in}\,C \land p\,{\in}\,E$.
\end{definition}

A stratified policy cannot generate infinite trajectories just
due to its form, without regard for the class of instances $\Q$
where it is applied, or the interpretation of the features; i.e.,

\begin{theorem}[Termination]
  \label{thm:termination1}
  Let $\pi$ be a rule-based policy over a set of features $\F$.
  If $\pi$ is stratified, $\pi$ is terminating.
\end{theorem}

\Omit{
  This concludes the definition of (simple) stratified policies where the monotonicity
  of each feature $f$ may depend at most on another feature $g$ of lower rank.
  As it is shown below, if the policy $\pi$ is stratified, then it is terminating.
}

\begin{example*}
  We illustrate a terminating policy for the Gripper domain which
  involves a robot that must move balls from room $A$ to room $B$,
  using two grippers.
  The rooms $A$ and $B$ are declared in the domain specification
  as \textbf{constants,} thus shared by all instances and identified
  with \emph{nominal concepts.}
  The policy $\pi$ is defined over the features:
  $A$ that tells whether the robot is in room $A$,
  $m$ that counts the number of balls being held, and
  $n$ that counts the number of objects in room $A$.
  The rules are:
  \begin{alignat*}{2}
    % (:rule (:conditions ) (:effects (:e_n_dec f16)))
    r_1\ :\ && \set{\GT{n}}\         &\mapsto\ \set{\DEC{n}, \UNK{m}}\,, \\
    % (:rule (:conditions ) (:effects (:e_n_dec f7) (:e_n_bot f16)))
    r_2\ :\ && \set{\GT{m}}\         &\mapsto\ \set{\DEC{m}}\,, \\
    % (:rule (:conditions (:c_n_gt f7)) (:effects (:e_n_dec f12) (:e_n_bot f7) (:e_n_bot f16)))
    r_3:\ \ && \set{A, \GT{m}}\      &\mapsto\ \set{\neg A}\,, \\
    % (:rule (:conditions (:c_n_eq f7)) (:effects (:e_n_inc f12) (:e_n_bot f7) (:e_n_bot f16)))
    r_4:\ \ && \set{\neg A, \EQ{m}}\ &\mapsto\ \set{A}\,.
  \end{alignat*}
  The first rule says to decrease the number of balls in room $A$
  (i.e., to pick them up) whenever possible, perhaps affecting $m$;
  the second to drop balls somewhere while not affecting $n$ (the only
  way to achieve this is to drop balls in room $B$);
  the third rule says to move from room $A$ to $B$ when holding a ball,
  and the last one to move from $B$ to $A$ when holding nothing.
  This is a general policy that moves all balls in room A to room B,
  and that works for any number of balls.

  The policy is stratified and thus terminating (cf.\ Def.~\ref{def:stratified} and Thm.~\ref{thm:termination1}).
  The feature $n$ is monotone in $\pi$ as no rule increases it,
  $m$ is monotone in $\pi$ \textbf{given} $n$ as it is monotone in $\varrho(\pi,n,0)\,{=}\,\varrho(\pi,n,1)\,{=}\,\set{r_2,r_3,r_4}$,
  and $A$ is monotone in $\pi$ \textbf{given} $m$ as it is monotone in $\varrho(\pi,m,0)\,{=}\,\set{r_1,r_4}$ and it is monotone in $\varrho(\pi,m,1)\,{=}\,\set{r_1,r_4}$.
  %A ranking that works orders the features as \tup{n,m,A}.
  \qed
\end{example*}

\subsection{$k$-Stratified Policies}

Stratified policies can be generalized to cases where the monotonicity of $f$
depends on multiple features of lower rank.
For this, we need to define when a feature $f$
is monotone given a \textbf{set} $G$ of features, and this requires
the consideration of Boolean valuations over the features. Recall
that a Boolean valuation assigns a Boolean value $\nu(g)$ in $\set{0,1}$
for every feature $g$, whether Boolean or numerical.
%For numerical features $g$,
%$\nu(g)=0$ in a state $s$ if $g(s)=0$, and $\nu(g)=1$ if $g(s) > 1$.

\begin{definition}[$k$-Conditional Monotonicity]
  \label{def:monotone:conditional:k}
  Let $\F$ be a set of features, let $R$ be a set of policy rules over $\F$,
  and let $f$ and $G$ be a feature and a set of features, respectively, mentioned in $R$.
  Then, $f$ is \textbf{monotone in $R$ given $G$} iff for each Boolean
  valuation $\nu$ for $G$, $f$ is monotone in $\varrho(R,G,\nu)$, where
  $\varrho(R,G,\nu) \doteq \cap \set{ \varrho(R,g,\nu(g)) \,{\mid}\, g\,{\in}\,G }$.
  %\begin{alignat*}{1}
  %  \varrho(R,G,\nu)\ &\doteq\ \cap \set{ \varrho(R,g,\nu(g)) \,{\mid}\, g\,{\in}\,G } \,.
  %\end{alignat*}
\end{definition}

As $G$ increases, more ``contexts'' $\varrho(R,G,\nu)$ need to be considered, but
each context is smaller. Thus, the chances for $f$ being monotone given $G$ increase
as $G$ contains more features.
Also, $f$ is monotone given $G'$ when it is monotone given $G$ and $G\,{\subseteq}\,G'$.
The definition of $k$-stratified policies can be similarly expressed as before: %is the natural generalization of Definition~\ref{def:stratified}:

\begin{definition}[$k$-Stratified Policies]
  \label{def:stratified:k}
  Let $\pi$ be a rule-based policy defined over the features in $\F$, and let $k$ be a positive integer.
  Then, $\pi$ is a \textbf{$k$-stratified policy} iff
  \begin{enumerate}[1.]
    \item Each rule in $\pi$ \textbf{entails the change} of some feature $f$, and
    \item There is ranking $\kappa$ for the features in $\pi$ such that
      \begin{enumerate}[{2}a.]
        \item If $\kappa(f)\,{=}\,0$, $f$ is \textbf{monotone in $\pi$,} and
        \item If $\kappa(f)\,{>}\,0$, there are features $G\,{=}\,\set{g_1,g_2,\ldots,g_\ell}$, with $\ell\,{\leq}\,k$,
          such that $\max\set{\kappa(g)\,{\mid}\,g\,{\in}\,G}\,{<}\,\kappa(f)$ and $f$ is \textbf{monotone in $\pi$ given $G$}.
      \end{enumerate}
  \end{enumerate}
\end{definition}

\Omit{% Not really needed...
  The class of all rule-based policies $\pi$ over the features in $\F$
  is denoted by $\Pi_\F$, and the class of all policies that are $k$-stratified
  is denoted by $\Pi^k_\F$.
}
It is not difficult to show that for any integer $k\,{>}\,1$,
there are policies that are $k$-stratified but not $(k-1)$-stratified,
and that there are policies that are terminating, but not $k$-stratified
for any $k$.
As before, $k$-stratified policies are terminating too:

\begin{theorem}[Termination of $k$-Stratified Policies]
  \label{thm:termination}
  Let $\pi$ be a rule-based policy defined on a set of features $\F$,
  and let $k$ be a positive integer.
  If $\pi$ is $k$-stratified, $\pi$ is terminating.
\end{theorem}
\begin{proof}[Proof (sketch).]
  Let $\pi$ be a $k$-stratified policy, and let $\kappa$ be a suitable ranking for $\pi$.
  %Let us suppose that $\pi$ is not terminating. We want to reach a contradiction.
  Let us suppose that $\pi$ is not terminating; i.e., there is an infinite trajectory
  of Boolean valuations over the features that is compatible with $\pi$.

  Since for each transition compatible with $\pi$, there is a feature that changes
  across the transition, there must be a feature $f$ in $\pi$ that is increased \textbf{and}
  decreased infinitely often in $\tau$.
  %
  %By definition of termination, and the fact that for each transition in $\tau$ there
  %is a feature that changes value across the transition (cf.\ condition 1 in Def.~\ref{def:stratified:k}),
  %there is an \textbf{infinite}
  %$\pi$-trajectory $\tau$ in the policy graph for $\pi$ and a feature $f$ in $\pi$
  %that is increased \textbf{and} decreased infinitely often in $\tau$.
  Let $f$ be such a feature of \textbf{minimum rank.}
  Clearly, $\kappa(f)$ cannot be zero as such features are monotone in $\pi$.
  Therefore, there is a set $G$ with at most $k$ features and $\max\set{\kappa(g)\,{\mid}\,g\in G}\,{<}\,\kappa(f)$
  such that $f$ is monotone in $\pi$ given $G$.
  This means that $f$ is monotone in $\varrho(\pi,G,\nu)$ for all the Boolean valuations $\nu$ of the features in $G$.
  Therefore, there are at least two valuations $\nu_0$ and $\nu_1$ for $G$ that appear infinitely often in $\tau$.
  Hence, there is some feature $g$ in $G$ whose Boolean value \textbf{flips} infinitely often in $\tau$, which
  implies that $g$ is increased \textbf{and} decreased infinitely often in $\tau$.
  This contradicts the choice of $f$ because $\kappa(g)\,{<}\,\kappa(f)$.
\end{proof}

At the same time, checking $k$-stratification is exponential in $k$ and not exponential
in the number of features like \textsc{sieve}. This is important because the notion of termination
captured by $k$-stratification for a low value of $k$ is most often powerful enough.
Indeed, our learning algorithm uses $k\,{=}\,1$.

\begin{theorem}[Testing $k$-Stratification]
  \label{thm:test}
  Let $\pi$ be rule-based policy defined on a set of features $\F$,
  and let $k$ be a positive integer.
  Testing whether $\pi$ is $k$-stratified can be done in time that
  is exponential in $k$, but polynomial in $|\pi|$ and $|\F|$,
  where $|\pi|$ refers to the number of rules in $\pi$.
\end{theorem}
\begin{proof}[Proof (sketch).]
  In order to check that $\pi$ is $k$-stratified, one needs to ``construct'' a suitable ranking $\kappa$.
  The construction proceeds in stages, first identifying the features of rank 0, then those of rank 1, etc.
  First, at stage 0, the monotone features $f$ in $\pi$ are identified in linear time and assigned $\kappa(f)\,{=}\,0$.
  Then, at stage $\ell$, each feature $f$ not yet ranked is tested whether there is a subset $G$ of at most $k$
  features of rank less than $\ell$ such that $f$ is monotone in $\pi$ given $G$.
  The number of subsets to try is $\O(n^k)$, where $n$ is the number of features in $\pi$, and for each such
  candidate, $\O(2^k)$ contexts must be considered.
  Hence, testing whether a new feature $f$ can be assigned a rank requires time that is exponential only in $k$.
  This check must be repeated $\O(n)$ times as there are $n$ features, and for at most $\O(n)$ stages.
\end{proof}

\Omit{
  \subsection{Main Properties of Stratified Policies}
  We finish this section by establishing two important properties of
  stratified policies.
  First, such policies are terminating. Second, testing whether
  a given policy $\pi$ is $k$-stratified can be done in time that
  is \textbf{exponential only in $k$.}
  Hence, testing whether a given policy $\pi$ is $k$-stratified,
  for fixed $k$, can be done in polynomial time.
  This is a sharp difference with the general notion of termination
  where testing a policy for termination requires space (and time)
  that is exponential in the number of features mentioned in $\pi$,
  as the policy graph for $\pi$ must be constructed \cite{refs}.
}

\section{Basic Learning Task and \textsc{Genex}}
\label{sect:learning}

We formulate the problem of learning generalized policies from examples (plans)
in two parts:

\begin{enumerate}[1.]
  \item \textbf{Basic learning task (BLT):} Given sets $\X^+$ and $\X^-$ of ``good'' and
    ``bad'' state transitions, respectively, and a feature pool $\F$, the task is to
    find a \textbf{stratified} policy $\pi$ over the features in $\F$ such that the good
    transitions are $\pi$-transitions, and no bad transition is a $\pi$-transition.
    In other words, the BLT task is about generalizing the good transitions, which may
    come from plans, while avoiding the bad transitions and ensuring termination.
  \item \textbf{Meta learning task (MLT):} Find the sets of good and bad transitions so
    that the BLT returns a policy that is \textbf{closed} and \textbf{safe} over all
    instances $P$ in a training set $\Q$.
\end{enumerate}

We focus on the basic learning task in this section and on the
meta learning task in the next one.

\Omit{
  from $\Q$ by iteratively
  following two steps:
  \begin{enumerate}[1.]
    \item Obtain a terminating (stratified) policy that generalizes a given set of
      positive and negative state transitions; provided by a planner/oracle (where
      positive transitions are those in the examples, and negative transitions are
      those leading to dead-end states).
    \item If the resulting policy is not closed in $\Q$, or new dead-end states
      are found (i.e., states that are not solvable by the planner/oracle), then
      extend the positive transitions with the planner, or the negative transitions
      with the discovered transitions to dead-end states, and go back to 1.
  \end{enumerate}
  The first step is achieved by learners for the \emph{basic learning task,} while
  the second step is achieved by a suitable \emph{wrapping} of such learners.

  Generalization beyond the instances in $\Q$ is obtained by the strong \textbf{inductive bias}
  imposed by the target class of policies, rule-based stratified policies, as such
  policies are not tied to the instances in $\Q$, and are terminating.
}

\subsection{Task and Algorithm}

The task is to learn a stratified policy that includes the good transitions,
and excludes the bad transitions; formally,

\begin{definition}[Basic Learning Task]
  \label{def:learning:basic}
  Let $\F$ be a pool of features, and let $\X^+$ and $\X^-$ be sets of transitions,
  called ``good'' and ``bad'' transitions, respectively.
  Then, $BLT(\F,\X^+,\X^-)$ is the task of finding a
  \textbf{stratified policy} $\pi$ over $\F$ such that $\X^+\subseteq\pi$ and $\X^-\cap\pi = \emptyset$.
\end{definition}

The BLT is defined in this way, leaving aside \textbf{closedness} and \textbf{safeness},
because it can be cast as a \textbf{hitting set problem} that admits efficient algorithms.

In general, a hitting set problem is the tuple $\tup{S,\H,c}$ where $S$ is a set of items,
$\H$ is a collection of $S$-subsets, and $c:S\rightarrow\mathbb{N}^+$ is a cost function.
The task is to find a min-cost subset $S'\,{\subseteq}\,S$ that ``hits'' every subset in $\H$
(i.e., $S'\cap S_i\neq\emptyset$ for $S_i\,{\in}\,\H$), where the cost of $S'$ is $c(S')\doteq \sum_{i\in S'} c(i)$.

\begin{definition}[Hitting Set Problem Induced by BLT]
  \label{def:hitting-set}
  Let $\F$ be a pool of features, and let $\X^+$ and $\X^-$ be sets of good and bad transitions.
  The hitting set problem $H(\F,\X^+,\X^-)$ is the tuple $\tup{\F,\H,c}$ where
  the cost function $c$ maps $f\,{\in}\,\F$ into its complexity, and $\H$
  consists of the following $\F$-subsets:
  \begin{enumerate}[$\bullet$]
    \item For each transition $(s,t)$ in $\X^+$, $\H$ contains the set
      $\set{f\,{\in}\,\F \,{\mid}\, \text{$f$ changes across $(s,t)$}}$.
    \item For each $(s,t)$ in $\X^-$ and each $(s',t')$ in $\X^+$, $\H$ contains
      $\set{f\,{\in}\,\F \,{\mid}\, \text{$f$ changes differently in $(s,t)$ and $(s',t')$}}$.
    \item For each pair $(s,s')$ of goal and non-goal states in the transitions in $\X^+$,
      $\H$ contains the set
      $\set{f\,{\in}\,\F \,{\mid}\, \text{the Boolean valuation of $f$ differs on $\set{s,s'}$}}$.
  \end{enumerate}
\end{definition}

The last type of subsets is not needed for solving the BLT.
However, policies that contain features that identify goal states tend to
generalize better over new unseen instances.

From a solution $\G$ for $H(\F,\X^+,\X^-)$, one can construct a policy
$\pi\,{=}\,\pi(\G,\X^+)$ whose rules \arule{C}{E} are obtained by \textbf{projecting}
the good transitions $(s,t)$ in $\X^+$ over the features in $\G$, like
in previous learning approaches.

\genex, depicted in Alg.~\ref{alg:greedy}, is a standard \textbf{greedy algorithm}
that solves $H(\F,\X^+,\X^-)$ by growing a hitting set $\G$. %, one feature at a time.
To guarantee completeness of the algorithm (cf.\ Theorem~\ref{thm:blt:1}),
the algorithm adds a \text{set} (chain) $C_f=\tup{f_0,f_1,\ldots,f_{k+1}\,{=}\,f}$
of features (of maximum score) that ends in $f$, and that provides complete
conditional monotonicity for $f$: $f_0$ is monotone for (the transitions in)
$\X^+$, and $f_{i+1}$ is monotone for $\X^+$ given $f_i$, for $i=0,1,\ldots,k$.
Stratification is guaranteed by ensuring that no choice
of chains create a circular ordering among the chosen features.
For this, each such chain $C_f$ imposes an ordering constraints $f_i\prec f_{i+1}$
that are maintained in the set $\Ord$ (line 9) and that is queried for selecting features (line 7).
%To ensure that the selected features in $\G$ are \textbf{stratified}
%for the transitions in $\X^+$, and hence that the resulting policy $\pi$
%is stratified, the algorithm does not add one feature $f$ of maximum score
%at a time.
%Instead, it adds a \textbf{set} (chain) $C_f=\tup{f_0,f_1,\ldots,f_{k+1}\,{=}\,f}$
%of features (of maximum score) that ends in $f$, and that provides complete
%conditional monotonicity for $f$, as $f_0$ is monotone for (the transitions in)
%$\X^+$, and $f_{i+1}$ is monotone for $\X^+$ given $f_i$, for $i=0,1,\ldots,k$.

\begin{algorithm}[t]
  \caption{\genex for solving $H(\F,\X^+,\X^-)$.}
  \label{alg:greedy}
  \begin{minipage}{\linewidth}
    \smallskip
    \textbf{Input:} Hitting set problem $H\,{=}\,H(\F,\X^+,\X^-)$. \\[2pt]
    \textbf{Output:} \FAIL, or hitting set $\G$ for $H(\F,\X^+,\X^-)$. %Solution $\G\subseteq\F$ (i.e., $\G\cap S\,{\neq}\,\emptyset$ for each $S\,{\in}\,H$).
    \smallskip
    \begin{algorithmic}[1]
      \State For each $f$ in $\F$, let $\cost(f) \gets \text{complexity}(f)$.
      \smallskip
      \State For each $f\,{\in}\,\F$, compute $C_f\,{=}\,\langle f_0, \ldots,f_{k+1}\,{=}\,f \rangle$ such that $f_0$
        is monotone in $\X^+$, $f_{i+1}$ is monotone in $\X^+$ given $f_i$, and $C_f$ is of \textbf{minimum cost,}
        where the cost of $C_f$ is $\sum_{i=0}^{k+1} \cost(f_i)$.
      \smallskip
      \State Let $\G:=\emptyset$
      \smallskip
      \State Let $\Ord:=\emptyset$ be the empty ordering of chosen features.
        Feature $f$ is \textbf{eligible} if $\Ord\,{\cup}\,\set{\tup{g,g'}\mid\tup{g,g'}\in C_f}$ is acyclic.
        The set of eligible features is denoted by $\mathcal{E}\,{=}\,\mathcal{E}(\F,\Ord)$.
      \smallskip
      \While{$\G$ is not a solution for $H$}
      \smallskip
      \State Let $f^*\,{\in}\,\mathcal{E}$ be a feature of \textbf{max} $\text{score}(f^*)$, where
      \smallskip
      \Statex\qquad $\text{score}(f)\,{=}\,|\{ S\,{\in}\,H \mid \G\cap S=\emptyset\ \land\ C_f \cap S\neq\emptyset\}|$
      \smallskip
      \Statex\quad\ \ divided by the cost of $C_f$.
      \smallskip
      \State If $\mathcal{E}\,{=}\,\emptyset$ or $\text{score}(f^*)\,{=}\,0$, \textbf{return} \FAIL
      \smallskip
      \State Let $\G \gets \G \cup C_{f^*}$
      \smallskip
      \State Let $\Ord \gets \Ord \cup \set{\tup{g,g'}\mid\tup{g,g'} \in C_{f^*}}$
      \smallskip
      \State Set $\cost(f)\gets 0$ for $f \in C_{f^*}$.
      \smallskip
      \State Recompute chains given new cost of features in $C_f$
      \smallskip
      \EndWhile
      \smallskip
      \State \textbf{return} $\G$
      \smallskip
    \end{algorithmic}
  \end{minipage}
\end{algorithm}

The greedy algorithm runs in low polynomial time in the
size of $H(\F,\X^+,\X^-)$. More interesting is that it is sound and
complete for the basic learning task:

\begin{theorem}[Soundness and Completeness of \genex]
  \label{thm:blt:1}
  Let $\F$ be a pool of features, and let $\X^+$ and $\X^-$ be sets of
  good and bad transitions, respectively.
  If \genex returns $\G\,{\subseteq}\,\F$ on input $H(\F,\X^+,\X^-)$, then
  \begin{enumerate}[1.]
    \item $\G$ is a hitting set for $H(\F,\X^+,\X^-)$, and
    \item the policy $\pi\,{=}\,\pi(\G,\X^+)$ obtained by projecting the transitions in $\X^+$ over $\G$ solves $\BLT(\F^+,\X^+,\X^-)$.
  \end{enumerate}
  Else, if \genex returns \FAIL on input $H(\F,\X^+,\X^-)$, then
  there is no solution $\pi$ for $\BLT(\F^+,\X^+,\X^-)$ whose features separate
  goal from non-goal states.
\end{theorem}
\begin{proof}[Proof (sketch).]
  \textbf{Soundness.}
  Let us assume that \genex returns $\G$. Clearly, $\G$ is a hitting set for $H(\F,\X^+,\X^-)$.
  Also, it is not hard to see that there is a ranking $\kappa$ that renders $\pi(\G,\X^+)$ stratified.
  Indeed, since $\Ord$ remains acyclic, there is such ranking $\kappa$ throughout the execution of
  the loop, at the start of each iteration.
  %Firstly, the empty ranking works for the empty $\G$.
  %Then, assuming that such a ranking exists for an iteration where the feature $f^*$
  %is selected, the ranking $\kappa$ is amended to work for $\G\cup C_{f^*}$
  %using the chain $C_{f^*}$, since the chain provides a complete ``monotone support'' for $f^*$.
  %\alert{*** Not fully clear: What if $C_{f^*}$ contains $(g,g')$ and another previously chosen chain contains $(g',g)$? ***}

  \textbf{Completeness.}
  Let $\pi$ be a stratified policy over $\G\,{\subseteq}\,\F$ such that $X^+\,{\subseteq}\,\pi$,
  $\X^-\cap\pi\,{=}\,\emptyset$, and $\G$ separates goal from non-goal states. Moreover, let $\pi$
  be such a policy with a \textbf{minimum} number of rules; i.e., each rule in $\pi$ is compatible
  with at least one transition in $\X^+$.
  One can show that during the execution of \genex, at the beginning of
  each iteration, there is a feature $f$ in $\G$ such that the score of $C_f$ is non-zero.
  Hence, \genex cannot terminate with failure.
\end{proof}

Finally, if $\Q$ is a finite collection of instances, then one can
construct sets $\X^+$ and $\X^-$ of good and bad transitions over the instances in $\Q$
such that any solution $\pi$ for $\BLT(\F,\X^+,\X^-)$ solves $\Q$.

\begin{theorem}
  \label{thm:blt:2}
  Let $\Q$ be a \textbf{finite} class of instances, let $\F$ be a pool of features, and let $\X^+$
  and $\X^-$ be set of good and bad transitions that satisfy the following:
  \begin{enumerate}[1.]
    \item For each instance $P$ in $\Q$, and each alive state $s$ in $P$, there is a
      transition $(s,s')$ in $\X^+$, and
    \item For each instance $P$ in $\Q$, $\X^-$ contains all the transitions $(s,s')$
      in $P$ where $s$ is alive and $s'$ is a dead-end state.
  \end{enumerate}
  If $\pi$ is a solution of $\BLT(\F,\X^+,\X^-)$, then $\pi$ solves $\Q$.
\end{theorem}
\begin{proof}
  Let $P$ be in $\Q$, and let $\tau=(s_0,s_1,\ldots,s,t)$ be a \textbf{maximal}
  $\pi$-trajectory in $P$. Clearly, $\tau$ is acyclic since $\pi$ is stratified.
  Since all transitions $(s,t)$ in $P$ from an alive state $s$ to a dead-end $t$
  are in $\X^-$, $t$ is not a dead-end state. On the other hand, $t$ cannot be alive as
  otherwise there would be a transition $(t,t')$ in $\X^+$, implying that $\tau$ is not
  maximal. Therefore, $t$ must be a goal state.
\end{proof}

\section{Meta Learning Task: \textsc{Wrapper}}

The basic learning task $\BLT(\F,\X^+,\X^-)$ is solved efficiently
by \genex, provided it has solution.
There is however an important open question: what state transitions
from the training instances to include in $\X^+$ and $\X^-$?

One answer to above question is given by sets $\X^+$ and $\X^-$
that comply with the conditions in Theorem~\ref{thm:blt:2}.
However, this is impractical as computing such sets requires
the expansion of the state space for the training instances in $\Q$,
and would result in very large sets of transitions, at least
one transition per each alive state in each instance.
In this section, we describe an efficient \textbf{wrapper algorithm},
simply called \wrapper, that starting from example paths computed by
a planner, identifies additional transitions that are added to $\X^+$ and $\X^-$.
The algorithm, being greedy, is incomplete, yet it is able to
solve a large number of benchmarks.

The idea behind \wrapper, depicted in Alg.~\ref{alg:wrapper}, is simple.
To find a policy that solves $\Q$, the wrapper works with a small subset
$\Q'\,{\subseteq}\,\Q$ (line 2), finds a solution $\pi$ for $\Q'$ using \genex (lines 4--7),
and tests $\pi$ on $\Q$ (line 8).
If $\pi$ solves $\Q$, it returns $\pi$. Else, $\Q'$ is updated, and the process repeats.

Finding a solution for $\Q'$ may involve multiple calls to \genex
with different sets $\X^+$ and $\X^-$ of good and bad transitions for $\Q'$.
If $H(\F,\X^+,\X^-)$ has solution $\G$, the policy $\pi\,{=}\,\pi(\G,\X^+)$
is executed on the instances in $\Q'$. Three different outcomes may arise:
1)~$\pi$ solves $\Q'$; 2)~an alive state $s$ is found where $\pi$ is not
defined, hence $\pi$ is not closed; or 3)~a dead-end state $t$ is reached,\footnote{Dead-end states
  are identified with the planner: state $t$ is declared as dead-end state
  iff the planner cannot find a plan for it.}
hence $\pi$ is not safe.
In case 2, $\X^+$ is extended with a transition $(s,t)$ obtained with the planner.
In case 3, $\X^-$ is extended with the last $\pi$-transition $(s,t)$ leading to the dead-end state.
%Initially, $\X^+_{\Q'}$ contains the transitions in the plans for
%the instances in $\Q'$, and $\X^-_{\Q'}\,{=}\,\emptyset$.
%Then, upon obtaining a solution $\pi$ for $\BLT(\F,\X^+_{\Q'},\X^-_{\Q'})$,
%$\pi$ is tested on $\Q'$: if $\pi$ reaches an alive state $s$ on which $\pi$
%cannot be applied, the planner is called to provide a new transition $(s,t)$
%to add to $\X^+_{\Q'}$, or if a dead-end\footnote{Dead-end states are detected
%  by calling the planner: state $t$ is declared as dead-end state iff
%  the planner cannot find a plan for $t$.}
%state $t$ is reached, the $\pi$-transition that leads to $t$ is added to $\X^-_{\Q'}$.
This is repeated until a solution for $\Q'$ is obtained, or
\genex fails, in which case \wrapper also fails (line 5).
This loop (lines 4--7) is called the \textbf{inner loop} of the wrapper.

The optional policy simplification in line 6 tries to remove conditions and
insert unknown effects (i.e., of type $\UNK{n}$) in the rules
of the projected policy $\pi\,{=}\,\pi(\G,\X^+)$ to increase its coverage.
The simplification is an iterative, greedy, process that at each iteration
attempts to remove a condition or insert an unknown effect while preserving
the \emph{invariant:} $\pi$ remains stratified under the same ranking $\kappa$,
$\pi$ has the same conditional monotonicites, and $\pi\,{\cap}\,\X^-\,{=}\,\emptyset$.
%\alert{*** Note specific simplification ***}

Finally, if the found policy, which solves $\Q'$, does not solve $\Q$,
two strategies for updating $\Q'$ are considered.
Both strategies work with a \emph{static ordering} $P_1,P_2,P_3,\ldots$ of the instances in $\Q$,
determined by the length of the plans computed by the planner $\Phi$, from larger to shorter plans.
In the \textbf{first strategy,} named $S_1$, $\Q'$ is always a singleton, initialized to \set{P_1}.
If $\pi$ solves $\Q'\,{=}\,\set{P_k}$ but not $P_\ell$, where $\ell$ is minimum,
$\Q'$ is updated to \set{P_\ell} if $\ell\,{>}\,k$, else to \set{P_{k+1}}.
%Clearly, the maximum number of ``outer iterations'' is bounded by $|\Q|$.
In the \textbf{second strategy,} named $S_2$, $\Q'$ also starts as \set{P_1}.
But, if $\pi$ solves $\Q'$, with max instance index $k$, but not $P_\ell$,
of minimum index, $\Q'$ is set to $\Q'\cup\set{P_\ell}$ if $\ell\,{<}\,k$, else to \set{P_\ell}.
The loop comprising lines 4--8 is referred to as the \textbf{outer loop.}
The number of iterations of the outer loop is bounded by $|\Q|$
for ther first strategy, and by $|\Q|^2$ for the second strategy.

\begin{algorithm}[t]
  \caption{\wrapper for \genex.}
  \label{alg:wrapper}
  \begin{minipage}{\linewidth}
    \smallskip
    \textbf{Input:} Training set $\Q$ with planning instances $P$, pool $\F$ of features, and planner $\Phi$. \\[2pt]
    \textbf{Output:} \FAIL, or stratified policy $\pi$ that solves $\Q$.
    \smallskip
    \begin{algorithmic}[1]
      \State Use the planner $\Phi$ on each instance $P$ in $\Q$ to obtain plan $\tau_P$ for $P$.
        Collect the transitions in $\tau_P$ into the sets $\X^+_P$ and $\X^-_P=\emptyset$ of good and bad transitions for $P$.
      \smallskip
      \State \textcolor{blue!70!black}{\mbracket{\text{Outer loop, lines 2--8}}} Get non-empty subset $\Q'\,{\subseteq}\,\Q$ of instances.
        If no more subsets $\Q'$ are available (see text), return \FAIL.
      \smallskip
      \State Let $\X^+\,{\gets}\,\cup\set{\X^+_P \,{\mid}\, P\,{\in}\,\Q'}; \X^-\,{\gets}\,\cup\set{\X^-_P \,{\mid}\, P\,{\in}\,\Q'}$.
      \smallskip
      \State \textcolor{blue!70!black}{\mbracket{\text{Inner loop, lines 4--7}}} Solve $H\,{=}\,H(\F,\X^+,\X^-)$ with \genex (Algorithm~\ref{alg:greedy}).
      \smallskip
      \State \textbf{Return} \FAIL, if \genex fails, else let $\pi\,{=}\,\pi(\G,\X^+)$ for the solution $\G$ of $H$.
      \smallskip
      \State \textcolor{blue!70!black}{\mbracket{\text{Optional}}} Simplify policy $\pi$ (see text).
      \smallskip
      \State Test $\pi$ on $\Q'$. If, for some $P$ in $\Q'$, there is a $\pi$-trajectory $\tup{s_0,\ldots,s,t}$
        such that $t$ is identified as a dead-end state, or $t$ is not covered by $\pi$, then: in the first case, augment
        $\X^-$ with transition $(s,t)$ (to be avoided), else, in the second case, augment $\X^+$ with a transition
        $(t,t')$ ``recommended'' by the planner $\Phi$ (to be covered). \textbf{Continue} to Step 4.
      \smallskip
      \State Test $\pi$ on $\Q$. If $\pi$ does not solve $P$ in $\Q$, update the subset $\Q'$ to contain $P$ (and possibly
        other instances), and \textbf{continue} to Step 2.
        Else, \textbf{return} $\pi$ as $\pi$ solves $\Q$.
        (See the text for strategies to update $\Q'$.)
      \smallskip
    \end{algorithmic}
  \end{minipage}
\end{algorithm}

\section{Experiments}
\label{sect:exp}

We implemented \genex and \wrapper in Python with the help of the libraries \dlplan and \tarski
\cite{drexler-et-al-dlplan2022,tarski}.
The planner is \siw \cite{nir:ecai2012} which is fast and effective, except for \DOMAIN{Spanner}
where \bfws \cite{nir:aaai2017} is used because \siw can only solve instances with one nut.
The source code, benchmarks, and results are publicly available.\footnote{\url{http://github.com/bonetblai/learner-policies-from-examples}.}
%%\alert{*** ZENODO? ***}

The other approaches for computing general rule-based policies for generalized planning are:
the approach of \citeau{frances:aaai2021} for computing general
policies, and the approach of \citeau{drexler:icaps2022} for computing sketches of bounded width.
Both approaches are based on SAT/ASP, and require the full expansion of the state-space
for the training instances, and thus they cannot handle large training instances or
feature pools.
Indeed, the reported experiments for the first approach, only consider 9 domains, with pools of
up to 2,000 features, and instances with up to 6,000 non-equivalent states. General policies
are obtained by a careful choice of the training instances.
The second approach reports experiments on 9 domains, from which policies are obtained for 6
domains. The pools involved have at most a thousand features, and tens of states.
The domains solved by either approach are solved by the new approach,
but not vice versa: there are domains solved by the new approach that cannot be solved by
any of the previous approaches.

We carried out experiments on 34 standard planning domains of various sorts that pose
different type of challenges for generalized planning. Some of these domains have been
considered before, and others are new (we don't have space to describe the domains).
In all cases, we pre-computed a pool of features using \dlplan with a complexity bound of 15
and max-depth bound of 5, except for Logistics and 8-puzzle domains where the bound was
increased to 20.
The \wrapper starts with strategy $S_1$, switching to strategy $S_2$ when the first fails, as explained below.

\subsection{Results}
\label{sect:exp:results}

We divide the results into 5 categories, depending on the number of considered plans, and
the addition of good and bad transitions, in order to obtain a policy $\pi$ that solves $\Q$:
\begin{enumerate}[C1.]
  \item 5 domains that require a single call to \genex (i.e., only the first plan is considered, and no additional transitions are needed).
  \item 4 domains that require just the first plan, but additional good or bad transitions are needed.
  \item 5 domains that require just one plan, but the first plan considered does not yield a solution.
  \item 6 domains that require the second strategy $S_2$.
  \item 14 domains where no general policy is found with the reason (only 3 timeouts, and 11 for lack of expressivity, explained below).
\end{enumerate}
Notice that there are 14 domains where general policies are found by generalizing a single plan, cf.\ categories C1--C3.

Table~\ref{table:1} shows the results for categories C1--C4, 20 domains, where
$|\Q|$ is the number of instances in the training set $\Q$, $|\S|$ is the number of
states considered (seen) during training, $|\F|$ is the size of the feature pool,
`Strat.' is the wrapper strategy, and `Outer' and `Inner' are the total number of
iterations for outer and inner loop of \wrapper, respectively.

The table also contains details for the last outer iteration: $|\Q'|$ is
the number of instances from $\Q$ used in the last outer iteration,
`Inner$^*$' is the number of inner iterations for the last outer iteration,
$|\X^+|$ and $|\X^-|$ is the last number of good and bad transitions,
$|H|$ is the size of the hitting set problem, $|\G|$ is the size of
hitting set, and $|\pi|$ is the number of rules in the projected
policy $\pi\,{=}\,\pi(\G,\X^+)$.

The last four columns in Table~\ref{table:1} show the accumulated times in seconds spent by
preprocessing to support \genex, \genex itself, verification, and total wall time.
The verification is costly when the number of instances in $\Q$ is large,
or the policy visits a large number of states.

\begin{table*}[t]
  \centering
  \resizebox{\linewidth}{!}{
  \begin{tabular}{@{}lcrrcccccccrccrrrr@{}}
    \toprule[1pt]
                                         &        &        &          &        &        &        & \multicolumn{7}{c}{Last (outer) iteration for \wrapper}            & \multicolumn{4}{c}{Time in seconds}            \\
    \cmidrule[.75pt](r{2pt}){8-14} \cmidrule[.75pt](l{2pt}){15-18}
      Domain                             & $|\Q|$ & $|\S|$ &   $|\F|$ & Strat. &  Outer &  Inner & $|\Q'|$ &Inner$^*$& $|\X^+|$ & $|\X^-|$ & $|H|$ & $|\G|$ & $|\pi|$ &    Prep. &   \genex &     Verif. &      Total \\
    \midrule[.75pt]
\Verb!Blocks4ops-clear!                  &     53 &    172 &   13,579 & $S_1$ &      1 &      1 &       1 &      1 &        7 &        0 &    14 &      2 &       2 &       0.63 &       0.01 &       0.18 &      10.40 \\
\Verb!Delivery-1pkg!                     &    169 &    981 &    5,765 & $S_1$ &      1 &      1 &       1 &      1 &        9 &        0 &    18 &      4 &       4 &       0.27 &       0.02 &       3.21 &      24.61 \\
\Verb!Gripper!                           &      5 &     60 &    6,154 & $S_1$ &      1 &      1 &       1 &      1 &       19 &        0 &    38 &      3 &       4 &       0.46 &       0.21 &       3.95 &       8.22 \\
\Verb!Reward!                            &    207 &  1,119 &  249,122 & $S_1$ &      1 &      1 &       1 &      1 &       17 &        0 &    34 &      2 &       2 &      26.16 &      33.73 &       6.97 &     152.51 \\
\Verb!Visitall!                          &    460 &  4,288 &  146,085 & $S_1$ &      1 &      1 &       1 &      1 &       20 &        0 &    40 &      2 &       3 &      15.68 &       1.41 &     991.55 &   1,110.15 \\
\midrule
\Verb!Childsnack!                        &     10 &     85 &    2,900 & $S_1$ &      1 &      2 &       1 &      2 &       15 &        0 &    30 &      6 &       8 &       0.50 &       0.08 &      32.17 &      35.64 \\
\Verb!Spanner-1nut!                      &     90 &    540 &    8,001 & $S_1$ &      1 &      2 &       1 &      2 &        6 &        1 &    19 &      3 &       3 &       0.52 &       0.08 &       0.58 &      11.75 \\
\Verb!Logistics-1truck!                  &     67 &    363 &  262,509 & $S_1$ &      1 &     17 &       1 &     17 &       32 &        0 &    64 &      5 &       8 &     905.52 &     803.52 &       2.71 &   1,926.61 \\
\Verb!Barman-1cocktail-1shot!            &     90 &  1,350 &   69,040 & $S_1$ &      1 &     21 &       1 &     21 &       32 &        2 &   133 &     11 &      22 &     243.14 &     539.91 &     216.19 &   1,049.92 \\
\midrule
\Verb!Blocks4ops-on!                     &     94 &    518 &  183,322 & $S_1$ &      2 &      2 &       1 &      1 &       10 &        0 &    20 &      4 &       7 &      24.24 &       1.90 &       1.14 &     108.27 \\
\Verb!Spanner!                           &    270 &  2,160 &   13,679 & $S_1$ &      2 &      3 &       1 &      2 &       10 &        1 &    31 &      3 &       3 &       1.96 &       0.39 &     176.08 &     210.42 \\
\Verb!Delivery!                          &    397 &  3,635 &   13,904 & $S_1$ &      4 &     21 &       1 &      3 &       23 &        0 &    46 &      4 &       5 &      42.18 &       6.69 &      23.03 &     135.30 \\
\Verb!Ferry!                             &    180 &  1,416 &    8,547 & $S_1$ &      5 &     13 &       1 &      6 &       17 &        0 &    34 &      4 &       5 &       8.69 &       3.95 &       3.66 &      39.36 \\
\Verb!Miconic!                           &    360 &  3,504 &  107,785 & $S_1$ &      9 &     30 &       1 &      6 &       20 &        0 &    40 &      4 &       5 &     253.81 &      73.78 &      50.87 &     502.87 \\
\midrule
\Verb!8puzzle-1tile-fixed!               &     18 &    140 &   11,230 & $S_2$ &      7 &     38 &       4 &      3 &       40 &        0 &    80 &      8 &      14 &      44.77 &      25.52 &       1.89 &      81.54 \\
\Verb!8puzzle-1tile!                     &     16 &    122 &   12,417 & $S_2$ &      8 &     52 &       3 &      6 &       40 &        0 &    80 &      8 &      18 &      69.14 &      35.62 &       2.20 &     117.80 \\
\Verb!Blocks4ops!                        &      5 &    129 &  100,897 & $S_2$ &     10 &     66 &       3 &      5 &       81 &        0 &   162 &      7 &      24 &   5,349.49 &   7,159.42 &      19.62 &  12,863.98 \\
\Verb!Sokoban-1stone-7x7!                &      8 &     67 &  115,214 & $S_2$ &     12 &    309 &       5 &     15 &       94 &        4 &   671 &     15 &      50 &  14,248.40 &  12,712.36 &     896.13 &  28,196.51 \\
\Verb!Logistics-1pkg!                    &     24 &    173 &  225,518 & $S_2$ &     15 &    138 &       4 &      8 &       56 &        0 &   112 &      7 &      16 &   6,936.58 &   5,913.87 &     416.71 &  16,136.66 \\
\Verb!Zenotravel-1plane!                 &     73 &    779 &   24,959 & $S_2$ &     28 &    266 &       7 &      3 &      122 &        0 &   244 &      7 &      18 &   2,039.20 &     277.37 &   2,905.90 &   5,406.45 \\
    \bottomrule[1pt]
  \end{tabular}}
  \caption{Results for the 20 domains in categories C1--C4 for which \wrapper yields a general policy. As it can be seen, \wrapper can handle hundreds of thousands of features,
    and instances with millions of reachable states (e.g., all instances in \protect\DOMAIN{Blocks4ops} have 10 blocks).
  }
  \label{table:1}
\end{table*}

The table is vertically divided into four parts corresponding to the categories C1--C4, respectively.
The 5 domains in the top are solved fast, in a few seconds, except \DOMAIN{Visitall}
that requires 18.5min, from which about 16.5min are spent in the verification over the 460
instances in $\Q$.
%%%
In the second part, there are domains with dead-end states, like \DOMAIN{Spanner-1nut},
which are used to augment the set $\X^-$ of bad transitions. Interestingly, just few
bad transitions are needed in domains with dead-end state: 1 for both versions of \DOMAIN{Spanner},
2 for \DOMAIN{Barman-1cocktail-1shot}, and 4 for \DOMAIN{Sokoban-1stone-7x7}.
%%%
The domains in the third part require just one example to find a general policy, but
such an example is not the first one considered as determined by a static ordering of
the instances.
In \DOMAIN{Blocks4ops-on}, for example, the second example suffices
without the need to consider additional good or bad transitions (i.e., $\text{Inner}^*\,{=}\,1$).
%%%
Finally, the domains at the bottom of the table required the strategy $S_2$ that
simultaneously considers transitions from more than one example path.
Strategy $S_2$ is used after $S_1$ considers all the $\Q'$ subsets (i.e., each singleton
$\Q'$ yields a policy that solve $\Q'$ but not $\Q$).
Regarding times, for the majority of the cases, a general policy is found in a few
minutes (13 domains finished in less than 10 minutes).
\DOMAIN{Sokoban-1stone-7x7} takes 7 hours and 50 minutes for 309 calls to \genex
in order to learn a policy that is able to solve 8 instances of Sokoban with 1 stone
on a $7\,{\times}\,7$ grid. The policy has 50 rules and 15 features.
It is a policy that is highly over fit to the training set.
%Only \DOMAIN{Zenotravel-1plane} required a considerable amount of time, about 5h.
%This is mainly due to the total number of inner iterations, and the time spent during verification.
%Only \DOMAIN{Blocks3ops} and
%\DOMAIN{Logistics-indexicals} required a considerable amount of time, about 3h and 19h,
%respectively. This is due to the large number of total inner iterations performed, and
%the time spent during verification.

For the remaining 14 domains, \wrapper was not able to find a general policy.
Table~\ref{table:2} summarizes the results, with columns similar to Table~\ref{table:1},
except for a new column entitled `Reason' that explains the failure of \wrapper:
`Edge' if there is a transition $(s,t)$ in $\X^+$ for which the pool contains no feature $f$ that
changes across $(s,t)$,
%`Deadend' if there are transitions $(s,t)$ and $(s',t')$ in $\X^+$ and $\X^-$,
%respectively, that cannot be separated by the features in the pool,
and `Timeout'
(when the process is killed and no timing data is available).
The `Edge' failure is due to \emph{lack of expressivity in the pool.}
We believe that for most cases, it would not be enough to just increase
the complexity bounds that are used to generate the pool of features.
Rather, it is simply that the features needed to express a general policy
fall outside the class of features that are captured by the grammar; i.e., features definable with 2-variable logic and
counting quantifiers.

\begin{table*}[t]
  \centering
  \resizebox{\linewidth}{!}{
  \begin{tabular}{@{}lcrrcccccccrcrrrr@{}}
    \toprule[1pt]
                                         &        &        &          &        &        &        & \multicolumn{6}{c}{Last (outer) iteration for \wrapper}            & \multicolumn{4}{c}{Time in seconds}           \\
    \cmidrule[.75pt](r{2pt}){8-13} \cmidrule[.75pt](l{2pt}){14-17}
      Domain                             & $|\Q|$ & $|\S|$ &   $|\F|$ & Strat. &  Outer &  Inner & $|\Q'|$ &Inner$^*$& $|\X^+|$ & $|\X^-|$ & $|H|$ &           Reason &    Prep. &   \genex &     Verif. &      Total \\
    \midrule[.75pt]
\Verb!Rovers!                            &    608 &  6,238 &  190,064 & $S_1$ &      1 &      1 &       1 &      1 &       25 &        0 &    50 &             Edge &      17.26 &       0.53 &       0.00 &     225.75 \\
\Verb!Tidybot-opt11-strips!              &      8 &    211 &   59,402 & $S_1$ &      1 &      1 &       1 &      1 &       40 &        0 &    80 &             Edge &       7.70 &       0.21 &       0.00 &     119.06 \\
\Verb!Tpp!                               &     11 &    608 &   14,128 & $S_1$ &      1 &      1 &       1 &      1 &      201 &        0 &   402 &             Edge &       9.32 &       0.21 &       0.00 &     237.19 \\
\Verb!8puzzle-2tiles!                    &     16 &    207 &    4,458 & $S_1$ &      1 &      2 &       1 &      2 &       20 &        0 &    40 &             Edge &       0.85 &       0.14 &       0.18 &       6.30 \\
\Verb!Hiking!                            &    180 &  1,215 &   16,723 & $S_1$ &      1 &      2 &       1 &      2 &        8 &        0 &    16 &             Edge &       1.41 &       0.06 &       4.02 &     190.05 \\
\Verb!Depot!                             &     18 &    851 &  255,079 & $S_1$ &      1 &     17 &       1 &     17 &      119 &        0 &   238 &             Edge &   4,771.50 &     379.59 &     415.29 &  15,511.73 \\
\Verb!Freecell!                          &     65 &  2,842 &  146,428 & $S_1$ &      1 &     35 &       1 &     35 &      133 &       19 & 2,964 &          Timeout &      -1.00 &      -1.00 &      -1.00 &      -1.00 \\
\Verb!Barman-1cocktail!                  &    270 &  3,998 &   69,040 & $S_1$ &      1 &     94 &       1 &     94 &      105 &        6 &   868 &             Edge &   3,056.59 &   5,474.44 &      36.43 &   8,769.37 \\
\Verb!Tetris-opt14-strips!               &     16 &    393 &   37,496 & $S_1$ &      1 &    142 &       1 &    142 &      180 &        1 &   550 &             Edge &   5,109.83 &   3,458.80 &   7,157.74 &  16,500.19 \\
\Verb!Satellite!                         &    950 &  6,716 &  171,475 & $S_1$ &      2 &    209 &       1 &    157 &      168 &        0 &   334 &          Timeout &      -1.00 &      -1.00 &      -1.00 &      -1.00 \\
\Verb!Driverlog!                         &    381 &  3,197 &  172,818 & $S_1$ &      6 &    152 &       1 &     16 &       31 &        0 &    62 &             Edge &   5,768.13 &   1,822.26 &      52.39 &   8,088.94 \\
\Verb!Zenotravel-1person!                &     80 &    500 &   42,271 & $S_1$ &     12 &    385 &       1 &     77 &       80 &        3 &   507 &             Edge &   2,262.94 &     350.48 &   1,359.22 &   4,701.44 \\
\Verb!Logistics!                         &    282 &  1,855 &   51,418 & $S_2$ &     11 &     72 &       4 &      5 &       99 &        0 &   198 &             Edge &     987.57 &     169.34 &      15.41 &   1,310.91 \\
\Verb!Blocks3ops!                        &    392 &  2,216 &  236,954 & $S_2$ &     17 &    102 &       7 &      8 &      106 &        0 &   206 &          Timeout &      -1.00 &      -1.00 &      -1.00 &      -1.00 \\
    \bottomrule[1pt]
  \end{tabular}
  }
  \caption{Results for the 14 domains in category C5 for which \wrapper is unable to find a general policy. The column `Reason' explains the failure: `Edge' means that there
    is good transition $(s,t)$ for which no feature in the pool changes across $(s,t)$, and `Timeout' means the solver did not finish after 12 hours.
  }
  \label{table:2}
\end{table*}

\subsection{Further Testing of Policies}

The learned policies solve the entire class $\Q$, even though, in many cases,
the number of instances seen during learning is a fraction of $\Q$.
In this section, we go further and test the policies on instances that are significantly
larger than the ones used for training.
For example, for \DOMAIN{Blocks4ops} %that contains general Blocksworld problems with 4 operators,
the training set contains 5 instances with 10 blocks each,
but we evaluate the resulting policy on instances with 20--45 blocks.

Table~\ref{table:3} shows statistics for this extended test. For each depicted domain,
the table contains the number of instances, the percentage of coverage (percentage of
solved instances), and the maximum and average \textbf{effective width}.

The \textbf{width of a sketch} $\pi$ on state $s$ for instance $P$ is the minimum
integer $k$ needed for the search algorithm IW$(k)$ to find a state $t$ from $s$
such that the pair $(s,t)$ is compatible with $\pi$ \cite{drexler:icaps2022}.
If the sketch $\pi$ is a policy, such state $t$ is a successor of $s$, and IW$(0)$ finds it.
Else, if $\pi$ is not a proper policy, it is regarded as a sketch and paired with IW$(k)$ for some $k\,{>}\,0$; $k\,{=}\,2$ in this test.
%Else, $\pi$ may still work if such a state $t$ can always be found by IW$(k)$, for some $k\,{>}\,0$.
%In our case, we use values $k\,{=}\,0,1,2$.
The max (resp.\ average) \textbf{effective width} %of $\pi$
for instance $P$ is the max (resp.\ average) width of $\pi$ for the states encountered when using $\pi$,
and the max (resp.\ average) effective width %of $\pi$
for the class $\Q$ is the maximum of the max (resp.\ average) effective width of $\pi$ over instances $P$ in $\Q$.
\Omit{
  We use \siwr \cire{siwr} to apply the learned policy $\pi$ over test instances.
  For a current state $s$, on a given test instance $P$, \siwr performs a bounded (IW)
  search from $s$ looking for a state $t$ such that the pair $(s,t)$ is compatible
  with $\pi$. The effort to find $t$ is measure by the \textbf{width} of the search:
  if $t$ is a successor, the effort is almost null and the width is $0$,
  otherwise, the width is a positive integer. We use a bound of 2 on the width for the search.
  The \textbf{effective width} of $\pi$ on instance $P$ refers to how difficult is
  to find the pairs $(s,t)$ when applying $\pi$ on $P$.
}

Values for coverage and effective width that deviate from 100\% and 0, respectively, are
highlighted in Table~\ref{table:3}, as they exhibit flaws of the policy $\pi$ when used on large instances.
Nonetheless, in all cases, except \DOMAIN{Childsnack} with a coverage of 41.7\%, the
learned policy solves the large instances in the test set.
Likewise, except for \DOMAIN{8puzzle}, \DOMAIN{Blocks4ops}, and \DOMAIN{Childsnack},
an effective width of 0 tells us that the learned policy is indeed a policy that selects
transitions that lead to goal states.
The case for the two version of the 8-puzzle is interesting.
The policies are learned using only instances for the 8-puzzle, but the test
set includes instances for the $(n^2-1)$-puzzle, $n\,{=}\,2,3,\ldots,6$.
The effective width is bigger than 0 \textbf{only} on some instances for $n$ in \set{4,5,6}.

\begin{table}
  \centering
  \resizebox{\linewidth}{!}{
  \begin{tabular}{@{}l rrrr@{}}
    \toprule[1pt]
                                         &         &        & \multicolumn{2}{c}{Effective width} \\
    \cmidrule[.75pt](r{2pt}){4-5}
      Domain                             & \#inst. & \%Coverage & max. & avg. \\
    \midrule[.75pt]
\Verb!8puzzle-1tile-fixed!               &  100 &  100.0\% & \cellcolor{blue!25}2.00 & \cellcolor{blue!25}0.50 \\
\Verb!8puzzle-1tile!                     &  100 &  100.0\% & \cellcolor{blue!25}2.00 & \cellcolor{blue!25}0.32 \\
\Verb!Blocks4ops-clear!                  &   30 &  100.0\% & 0.00 & 0.00 \\
\Verb!Blocks4ops-on!                     &   30 &  100.0\% & 0.00 & 0.00 \\
\Verb!Blocks4ops!                        &   30 &  100.0\% & \cellcolor{blue!25}2.00 & \cellcolor{blue!25}0.05 \\
\Verb!Childsnack!                        &  120 & \cellcolor{blue!25}  41.7\% & \cellcolor{blue!25}1.00 & \cellcolor{blue!25}0.20 \\
\Verb!Delivery!                          &   30 &  100.0\% & 0.00 & 0.00 \\
\Verb!Ferry!                             &   30 &  100.0\% & 0.00 & 0.00 \\
\Verb!Gripper!                           &   30 &  100.0\% & 0.00 & 0.00 \\
\Verb!Logistics-1pkg!                    &  420 &  100.0\% & 0.00 & 0.00 \\
\Verb!Logistics-1truck!                  &   55 &  100.0\% & 0.00 & 0.00 \\
\Verb!Miconic!                           &   31 &  100.0\% & 0.00 & 0.00 \\
\Verb!Reward!                            &   30 &  100.0\% & 0.00 & 0.00 \\
\Verb!Zenotravel-1plane!                 &  180 &  100.0\% & 0.00 & 0.00 \\
    \bottomrule[1pt]
  \end{tabular}
  }
  \caption{Coverage and effective width of some learned policies on large instances. %with a number of objects larger than the one used for training.
    Values that reveal flaws of the policies when applied on large instances are highlighted.
    The few number of highlighted cells shows that the learned policies are robust on
    instances that are significantly larger than the ones in the training sets.
  }
  \label{table:3}
\end{table}

\section{Conclusions}
\label{sect:conclusions}

We have introduced a novel formulation and algorithms for learning
generalized policies from examples computed by a planner.
In many cases, we have shown that the method yields general policies
by generalizing a single plan. In other cases, a few plans need to
be considered.

Two key contributions in relation to existent methods are that the
new method scales up to much larger pools of features and training instances,
and this enables the solution of domains that could not be addressed before.
At the same time, this is the first method in which the resulting general
policies are \emph{acyclic by design.}
This is achieved through the introduction of a new powerful structural
termination criterion that is built-in in the selection of the features.

The new learning algorithm is made up of a core algorithm, \genex, that is
implemented by a fast and efficient hitting-set algorithm, and the \wrapper
around it, that feeds \genex with transitions in $\X^+$ to be included in
the policy, and others in $\X^-$ to be excluded.
A shortcoming of \wrapper is that it is not complete: there could be positive
and negative transitions that yield a general policy over the target class of
problems, but the wrapper may fail to find them.
An improved, complete and efficient, wrapper around the complete and efficient
\genex is left as a challenge for future work.

\vfill
\pagebreak

\section*{Acknowledgments}

The research of H.\ Geffner has been supported by the Alexander von Humboldt Foundation with
funds from the Federal Ministry for Education and research, by the European Research Council (ERC),
Grant agreement No.\ 885107, and by the Excellence Strategy of the Federal Government and the NRW L\"ander, Germany.

\bibliographystyle{kr}
\bibliography{control}

\end{document}